\documentclass[conference]{IEEEtran}
\IEEEoverridecommandlockouts
\usepackage{fancyhdr}
\usepackage{amsmath,amssymb,amsfonts, amsthm}
\usepackage{algorithm}
\usepackage{algpseudocode}
\usepackage{multirow}
\usepackage{multicol}
\usepackage{cite}
\usepackage{graphicx}
\usepackage{textcomp}
\usepackage{xcolor}
\usepackage{booktabs}
\usepackage{tikz}
\usepackage{graphicx}
\usepackage{subcaption}
\usetikzlibrary{positioning, arrows.meta, calc}

\def\BibTeX{{\rm B\kern-.05em{\sc i\kern-.025em b}\kern-.08em
    T\kern-.1667em\lower.7ex\hbox{E}\kern-.125emX}}

\newtheorem{lemma}{Lemma}[section]

\begin{document}

\fancypagestyle{firststyle}
{
    \fancyhf{}
    \fancyhead[C]{\small This work has been submitted to the IEEE for possible publication. Copyright may be transferred without notice, after which this version may no longer be accessible.}
    \renewcommand{\headrulewidth}{0pt}
}

\title{Critically-Damped Third-Order Langevin Dynamics\\

}

\author{

\IEEEauthorblockN{Benjamin Sterling}
\IEEEauthorblockA{\textit{Department of Applied Math \& Statistics} \\
\textit{Stony Brook University} \\
Stony Brook, NY, USA \\
\texttt{benjamin.sterling@stonybrook.edu}}

\and

\IEEEauthorblockN{M\'{o}nica F. Bugallo}
\IEEEauthorblockA{\textit{Department of Electrical and Computer Engineering} \\
\textit{Stony Brook University} \\
Stony Brook, NY, USA \\
\texttt{monica.bugallo@stonybrook.edu}}
}

\maketitle
\thispagestyle{firststyle}

\begin{abstract}

While systems analysis has been studied for decades in the context of control theory, it has only been recently used to improve the convergence of Denoising Diffusion Probabilistic Models. This work describes a novel improvement to Third-Order Langevin Dynamics (TOLD), a recent diffusion method that performs better than its predecessors. This improvement, abbreviated TOLD++, is carried out by critically damping the TOLD forward transition matrix similarly to Dockhorn's Critically-Damped Langevin Dynamics (CLD). Specifically, it exploits eigen-analysis of the forward transition matrix to derive the optimal set of dynamics under the original TOLD scheme. TOLD++ is theoretically guaranteed to converge faster than TOLD, and its faster convergence is verified on the Swiss Roll toy dataset and CIFAR-10 dataset according to the FID metric.

\end{abstract}

\begin{IEEEkeywords}
TOLD, CLD, DDPMs, Langevin Dynamics, Critical Damping
\end{IEEEkeywords}

\section{Introduction}

A fundamental task in probabilistic machine learning is finding a method of transporting one probability distribution to another. The reason for doing so is that the distribution of real world data can be difficult to sample from, a task critical to Generative AI. In order to generate samples from a data distribution of interest, it is much easier to sample from a tractable distribution, and transfer those samples into the target data distribution. This is the strategy taken by Diffusion \cite{sohl2015deep} and Normalizing Flows \cite{normalizingflowsoriginal}, developed around the same time. Using these approaches, the latent distribution is taken to be the normal distribution because this choice regularizes the space, as the distance between different images is Euclidean.

More recently, the use of Denoising Diffusion Probabalistic Models (DDPMs) \cite{diffusiondenoising} and Normalizing Flows have exploded due to newfound applications and superiority over Generative Adversarial Networks (GANs). Many recent works \cite{jko, guth2022wavelet, subspacediffusion} seek to improve the data distribution's convergence to an equilibrium normal distribution because the faster the process converges, the closer the latent space is to being truly normally distributed. Continuous Diffusion Models \cite{diffusioncts} are typically modelled with Ornstein-Uhlenbeck Stochastic Processes. Several works add Hamiltonian Dynamics onto Ornstein-Uhlenbeck processes \cite{dockhorn2021score, doucet2022scorebased} to improve latent convergence by smoothing diffusion trajectories, as the noise is not directly added to the data variable. Third-Order Langevin Dynamics (TOLD) \cite{hold}, one such methodology that incorporates a Langevin Diffusion process, has already found use in achieving state-of-the-art results in tasks of voice generation \cite{shi2024langwave} and noisy image restoration \cite{shi2024noisy}. 

The main contribution of this paper is TOLD++, a new methodology to perform TOLD that exhibits universally faster convergence along diffusion time at no extra computational cost. We do so by considering the forward dynamics of the diffusion process as a classical system. The goal is to find parameter values that cause the transition matrix to become critically damped, i.e., possess a single eigen value with geometric multiplicity of 1. In the case of second-order dynamics, critical damping provides the best trade-off between convergence speed and undesirable oscillatory behavior. While this result is well established from classical signal processing for second order systems, it does not apply as obviously for third order systems parameterized like TOLD is. The novelty of this method is clearly demonstrated by faster convergences on both synthetic and real data in Figure \ref{fig:convergence} and Tables \ref{tab:fid_comparison} and \ref{tab:fid_comparison_cifar}. This paper is structured as follows: Section \ref{sec:problemformulation} reviews the TOLD framework, Section \ref{sec:methodology} describes the improvement built into TOLD++, Section \ref{sec:results} details the superior performance of TOLD++, and Section \ref{sec:conclusion} is the conclusion.

\section{Problem Formulation}
\label{sec:problemformulation}

The goal of DDPMs is to draw samples from some intractable data distribution, often high-dimensional, and find forward and backward processes that transition the samples to and from a normal distribution, respectively. There are many uses, a popular application being conditional image generation \cite{ho2022classifierfreediffusionguidance}. TOLD is an example of adding Hamiltonian Dynamics onto an Ornstein-Uhlenbeck process. The TOLD forward stochastic differential equations (SDEs) are the following:

\begin{equation}
\begin{aligned}
\begin{cases}
d\mathbf{q}_t &= \mathbf{p}_t \, dt, \\
d\mathbf{p}_t &= \big(-\mathbf{q}_t + \gamma \mathbf{s}_t \big) \, dt, \\
d\mathbf{s}_t &= \big(-\gamma \mathbf{p}_t - \xi \mathbf{s}_t \big) \, dt + \sqrt{2\xi L^{-1}} \, d\mathbf{w}.
\end{cases}
\end{aligned}
\label{eq:forwardsdes}
\end{equation}

where $\mathbf{q_t}$ is the data variable, $\mathbf{p_t}$ and $\mathbf{s_t}$ are auxilliary variables that represent momentum and acceleration respectively, $\mathbf{w}$ is a standard multidimensional Wiener process, $L$ is the Lipschitz constant of the potential function $U(\mathbf{q_t})$, in this work: $U(\mathbf{q_t}) = \frac{L}{2}|| \mathbf{q_t} ||^2$, $\gamma$ is a friction coefficient, and $\xi$ is an algorithmic parameter. The reverse SDEs are given by \cite{anderson1982reverse}:

\begin{equation}
\begin{cases}
d\mathbf{q}_t = -\mathbf{p}_t dt, \\
d\mathbf{p}_t = \left(\mathbf{q}_t - \gamma \mathbf{s}_t \right) dt, \\
d\mathbf{s}_t = \left(\gamma \mathbf{p}_t + \xi \mathbf{s}_t + 2\xi L^{-1} \nabla_{\mathbf{s}}\log p_{T-t}(\mathbf{x}) \right) dt + \sqrt{2\xi L^{-1}} d\mathbf{\Bar{w}}.
\end{cases}
\end{equation}

\noindent where $\mathbf{x_t} = (\mathbf{q_t}, \mathbf{p_t}, \mathbf{s_t})^T$ and $\nabla_{\mathbf{s_t}}\log p_{T-t}(\mathbf{x_t})$ is the score function. To simplify the notation, define the matrices $\mathbf{F}$ and $\mathbf{G}$ as:

\begin{equation}
    \mathbf{F} = \begin{bmatrix}
0 & 1 & 0 \\
-1 & 0 & \gamma \\
0 & -\gamma & -\xi
\end{bmatrix}, \mathbf{G} = \begin{bmatrix}
0 & 0 & 0 \\
0 & 0 & 0 \\
0 & 0 & \sqrt{2\xi L^{-1}}
\end{bmatrix}.
\end{equation}

Then the entire set of forward equations may simplify to:

\begin{equation}
    d\mathbf{x}_t = \mathbf{F}_k\mathbf{x}_tdt + \mathbf{G}_kd\mathbf{w}
    \label{eq:matrixforward},
\end{equation}

\noindent where $\mathbf{F}_k = \mathbf{F} \otimes \mathbf{I}_d, \mathbf{G}_k = \mathbf{G} \otimes \mathbf{I}_d$ and $d$ is the dimension of $\mathbf{q}_t$. If optimal $\gamma, \xi$ exist, with respect to convergence speed, and if it is possible to critically damp this system so that $\mathbf{F}$ has a single, simple eigen value, what choices are optimal?

\section{Methodology}
\label{sec:methodology}

The speed of the SDE's convergence is completely defined by the eigenvalues of $\mathbf{F}$, specifically the largest one, assuming that they are all real-valued. It can be verified by taking the characteristic polynomial of $\mathbf{F}$, that a repeated eigenvalue of $-\sqrt{3}$ (with geometric multiplicity of 1) is achieved if $\gamma = 2\sqrt{2}$ and $\xi = 3\sqrt{3}$. Thus $\mathbf{F}$ and $\mathbf{G}$ become:

\begin{equation}
    \mathbf{F} = \left[\begin{smallmatrix}
    0 & 1 & 0 \\
    -1 & 0 & 2\sqrt{2} \\
    0 & -2\sqrt{2} & -3\sqrt{3}
    \end{smallmatrix}\right], \quad
    \mathbf{G} = \left[\begin{smallmatrix}
    0 & 0 & 0 \\
    0 & 0 & 0 \\
    0 & 0 & 3^{1/4}\sqrt{6 L^{-1}}
    \end{smallmatrix}\right].
\end{equation}

This leads to faster convergence than TOLD's parameter choice of $\gamma = \sqrt{10}, \xi = 6$, that results in a largest eigenvalue of $-1$. It is proven in Lemma \ref{lem:optimal_eigenvalue} that a largest eigenvalue of $-\sqrt{3}$ cannot be further minimized with different choices of $\gamma$ and $\xi$.

\begin{lemma}
\label{lem:optimal_eigenvalue}
Suppose $\lambda_1, \lambda_2, \lambda_3 \in \mathbb{R}$ satisfy $p_F(\lambda) = 0$ and $\lambda_1 \leq \lambda_2 \leq \lambda_3$. It follows that $\min(\lambda_3) = -\sqrt{3}.$
\end{lemma}

\begin{proof}
The characteristic polynomial of $\mathbf{F}$ is:
\[p_F(\lambda) = \lambda^3 + \xi\lambda^2 + (\gamma^2 + 1)\lambda + \xi = (\lambda-\lambda_1)(\lambda-\lambda_2)(\lambda-\lambda_3).\]

Observe that the $\lambda^2$ and constant coefficients of $p_F$ are both $\xi$, and therefore equivalent. After expanding, this implies that the constant term $-\lambda_1\lambda_2\lambda_3$ is equal to the quadratic term $-(\lambda_1 + \lambda_2 + \lambda_3)$, thus: $\lambda_1\lambda_2\lambda_3 = \lambda_1 + \lambda_2 + \lambda_3$. We now proceed by showing that if $\lambda_3 < -\sqrt{3}$ there is a contradiction. Suppose $\epsilon_2, \epsilon_3 \in \mathbb{R}^+$ where $0 < \epsilon_3 < \epsilon_2$ such that $\lambda_3 = -\sqrt{3} - \epsilon_3$ and $\lambda_2 = -\sqrt{3} - \epsilon_2$. Thus we can rewrite $\lambda_1$ as

\[\lambda_1 = -(\lambda_2 + \lambda_3) + \lambda_1\lambda_2\lambda_3\]
\[\lambda_1 = -(\lambda_2 + \lambda_3) + 3\lambda_1 - \left( (\epsilon_2 + \epsilon_3)\sqrt{3} + \epsilon_2\epsilon_3\right)(-\lambda_1)\]

$\lambda_1 < 0$ because $\lambda_3 < -\sqrt{3}$ by supposition. Thus:

\[\lambda_1 < -(\lambda_2 + \lambda_3) + 3\lambda_1 \to \lambda_1 > \frac{\lambda_2 + \lambda_3}{2}\]

which violates $\lambda_1 \leq \lambda_2 \leq \lambda_3$. This is a contradiction.

\end{proof}

The remaining problem is that the original derivation of TOLD \cite{hold} utilizes Putzer's Lemma to derive the score matching algorithm \cite{putzer}, that does not apply to non-diagonalizable matrices. However, we use the following identity in the case of a repeated eigenvalue ($\lambda = -\sqrt{3}$):
\begin{equation}
    \exp(\mathbf{F}t) = \exp(\lambda t)\bigg[\mathbf{I} + t(\mathbf{F} - \lambda \mathbf{I}) + \frac{t^2}{2}(\mathbf{F} - \lambda \mathbf{I})^2 \bigg]
    \label{eq:nilpotent}
\end{equation}
This is easy to show by recognizing that $(\mathbf{F} - \lambda \mathbf{I})$ is nilpotent and recalling the Taylor series definition used in matrix functional calculus. The rest of the derivation proceeds by utilizing the Fokker-Planck Equations and stochastic calculus \cite{sarkka2019applied} applied to (\ref{eq:matrixforward}) to derive the forward distribution, $q(\mathbf{x_t}|\mathbf{x_0}) = \mathcal{N}(\boldsymbol{\mu}_t, \mathbf{\Sigma}_t)$:
\begin{equation}
    \frac{d \boldsymbol{\mu}_t}{dt} = \mathbf{F}_k\boldsymbol{\mu}_t
    \label{eq:mean_eq}
\end{equation}
\begin{equation}
    \frac{d \mathbf{\Sigma}_t}{dt} = \mathbf{F}_k\mathbf{\Sigma}_t + (\mathbf{F}_k\mathbf{\Sigma}_t)^T + \mathbf{G}_k\mathbf{G}_k^T.
    \label{eq:var_eq}
\end{equation}

Identically to ordinary differential equations, the solution to (\ref{eq:mean_eq}) is $\boldsymbol{\mu}_t = \left( \exp(\mathbf{F}t) \otimes \mathbf{I}_d \right) \mathbf{x}_0$, or using (\ref{eq:nilpotent}):
\begin{equation}
    \boldsymbol{\mu}_t= e^{-\sqrt{3} t} \left(\bigg[\mathbf{I} + t(\mathbf{F} + \sqrt{3} \mathbf{I}) + \frac{t^2}{2}(\mathbf{F} + \sqrt{3} \mathbf{I})^2 \bigg]\otimes \mathbf{I}_d \right)\mathbf{x}_0.
\end{equation}
Define:
\begin{equation}
    \exp(\mathbf{F}t) = \begin{bmatrix}
f_{11}(t) & f_{12}(t) & f_{13}(t) \\
f_{21}(t) & f_{22}(t) & f_{23}(t) \\
f_{31}(t) & f_{32}(t) & f_{33}(t)
\end{bmatrix} \nonumber \\ 
\end{equation}
\begin{equation}
= e^{-t\sqrt{3}} \begin{bmatrix}
t^2 + t\sqrt{3} + 1 & t^2\sqrt{3} + t & t^2\sqrt{2} \\
-t^2\sqrt{3} - t & -3t^2 + t\sqrt{3} + 1 & -t^2\sqrt{6} + 2t\sqrt{2} \\
t^2\sqrt{2} & t^2\sqrt{6} - 2t\sqrt{2} & 2t^2-2t\sqrt{3} + 1
\end{bmatrix}. \nonumber
\end{equation}

It follows that

\begin{equation}
    \boldsymbol{\mu}_t = \begin{bmatrix}
f_{11}(t) \mathbf{q_0} + f_{12}(t) \mathbf{p_0} + f_{13}(t) \mathbf{s_0} \\
f_{21}(t) \mathbf{q_0} + f_{22}(t) \mathbf{p_0} + f_{23}(t) \mathbf{s_0} \\
f_{31}(t) \mathbf{q_0} + f_{32}(t) \mathbf{p_0} + f_{33}(t) \mathbf{s_0}
\label{eq:meanupdate}
\end{bmatrix}.
\end{equation}

Following the steps in \cite{hold} and (\ref{eq:var_eq}), it is assumed $\mathbf{\Sigma}_0 = \mathrm{diag}(\Sigma_0^{qq}, \Sigma_0^{pp}, \Sigma_0^{ss}) \otimes \mathbf{I}_d$, where $\Sigma_0^{qq}, \Sigma_0^{pp}, \Sigma_0^{ss}$ are initial covariances; in this work they are taken to be $\alpha L^{-1}$. $\mathbf{\Sigma}_t$ is solved as follows:
\[\mathbf{\Sigma}_t = \exp(\mathbf{F}_kt) \mathbf{\Sigma}_0\exp(\mathbf{F}_kt)^T + \int_0^t \exp(\mathbf{F}_ks)\mathbf{G}_k\mathbf{G}_k^T \exp(\mathbf{F}_ks)^T ds.\]

This update simplifies significantly, identically to the appendix in \cite{hold}:

%\begin{scriptsize}
\begin{align}
    \scriptsize
    \Sigma_t^{qq} &= \sum_{j=1}^3 f^2_{1j}(t)\Sigma_0^{jj} + 6\sqrt{3}L^{-1}\int_0^t f_{13}^2(s)ds, \nonumber \\
    \Sigma_t^{qp} &= \sum_{j=1}^3 f_{1j}(t)f_{2j}(t)\Sigma_0^{jj} + 6\sqrt{3}L^{-1}\int_0^t f_{13}(s)f_{23}(s)ds, \nonumber \\
    \Sigma_t^{qs} &= \sum_{j=1}^3 f_{1j}(t)f_{3j}(t)\Sigma_0^{jj} + 6\sqrt{3}L^{-1}\int_0^t f_{13}(s)f_{33}(s)ds, \nonumber \\
    \Sigma_t^{pp} &= \sum_{j=1}^3 f^2_{2j}(t)\Sigma_0^{jj} + 6\sqrt{3}L^{-1}\int_0^t f_{23}^2(s)ds, \nonumber \\
    \Sigma_t^{ps} &= \sum_{j=1}^3 f_{2j}(t)f_{3j}(t)\Sigma_0^{jj} + 6\sqrt{3}L^{-1}\int_0^t f_{23}(s)f_{33}(s)ds, \nonumber \\
    \Sigma_t^{ss} &= \sum_{j=1}^3 f^2_{3j}(t)\Sigma_0^{jj} + 6\sqrt{3}L^{-1}\int_0^t f_{33}^2(s)ds,
    \label{eq:covarianceupdate}
\end{align}

\begin{equation}
\mathbf{\Sigma}_t = \begin{bmatrix}
\Sigma_t^{qq} & \Sigma_t^{qp} & \Sigma_t^{qs} \\
\Sigma_t^{qp} & \Sigma_t^{pp} & \Sigma_t^{ps}  \\
\Sigma_t^{qs} & \Sigma_t^{ps} & \Sigma_t^{ss} 
\end{bmatrix} \otimes \mathbf{I}_d, \nonumber
\end{equation}

\noindent where $\Sigma^{11}_0 = \Sigma^{qq}_0, \Sigma^{22}_0 = \Sigma^{pp}_0, \Sigma^{33}_0 = \Sigma^{ss}_0$. The integrals are all analytically solvable because each $f$ involves a polynomial of $t$ multiplied by an exponential of a constant times $t$. Just like the original TOLD, it can be proven that:

\[\lim_{t \to \infty} \Sigma_t^{qq}=\lim_{t \to \infty} \Sigma_t^{pp}=\lim_{t \to \infty} \Sigma_t^{ss} = L^{-1}.\]
\[\lim_{t \to \infty} \Sigma_t^{qp}=\lim_{t \to \infty} \Sigma_t^{qs}=\lim_{t \to \infty} \Sigma_t^{ps} = 0.\]
\[\lim_{t \to \infty}\boldsymbol{\mu}_t = \mathbf{0}.\]

\begin{algorithm}
\caption{TOLD/TOLD++ Training Algorithm}\label{alg:TOLD}
\begin{algorithmic}[1]
    \Require Input data $\mathbf{q}_0$, $\Sigma_0 \in \mathbb{R}^{3\times3}$, Score Network $\mathfrak{S}$
    \For{$k = 1$ to $n_{train}$}
        \State $\mathbf{p}_0 \gets \mathcal{N}(\mathbf{0}, \frac{\alpha}{L} \mathbf{I}_d), \mathbf{s}_0 \gets \mathcal{N}(\mathbf{0}, \frac{\alpha}{L} \mathbf{I}_d)$    
        \State $t \gets \mathcal{U}(0, T)$
        \State Calculate $f_{ij}(t)$ for $1 \leq i, j \leq 3$
        \State Perform Equations (\ref{eq:meanupdate}) and (\ref{eq:covarianceupdate})
        \State Take Cholesky Decomposition, $L_t$, of $\Sigma_t = $ \hfill 
$\begin{bmatrix}
\Sigma_t^{qq} & \Sigma_t^{qp} & \Sigma_t^{qs} \\
\Sigma_t^{qp} & \Sigma_t^{pp} & \Sigma_t^{ps} \\
\Sigma_t^{qs} & \Sigma_t^{ps} & \Sigma_t^{ss}
\end{bmatrix}$
        \State $\boldsymbol{\epsilon_1}, \boldsymbol{\epsilon_2}, \boldsymbol{\epsilon_3} \sim \mathcal{N}(\mathbf{0}, \mathbf{I}_d)$
        \vspace*{0.5em}
        \State $\mathbf{z}_t = \begin{bmatrix}
\mathbf{q}_t \\
\mathbf{p}_t \\
\mathbf{s}_t
\end{bmatrix} \gets \boldsymbol{\mu}_t + \begin{bmatrix}
L_t^{qq} \boldsymbol{\epsilon_1} \\
L_t^{pq}\boldsymbol{\epsilon_1} + L_t^{pp}\boldsymbol{\epsilon_2} \\
L_t^{sq}\boldsymbol{\epsilon_1} + L_t^{sp}\boldsymbol{\epsilon_2} + L_t^{ss}\boldsymbol{\epsilon_3}
\end{bmatrix}$
    \vspace*{0.5em}
\State $l_t \gets \Bigg( \Sigma_t^{ss} - \frac{(\Sigma_t^{sq})^2}{\Sigma_t^{qq}} - \bigg( \Sigma_t^{pp} - \frac{(\Sigma_t^{pq})^2}{\Sigma_t^{qq}} \bigg)^{- 1} \bigg( \Sigma_t^{sp} - \frac{\Sigma_t^{sq}\Sigma_t^{pq}}{\Sigma_t^{qq}}\bigg)^2 \Bigg)^{-1/2}$
    \vspace*{0.5em}
    
    \State $\mathbf{s}_{\theta} \gets \mathfrak{S}(\mathbf{z}_t, t)$ \Comment{$\theta$ are score network parameters}

    \vspace*{0.5em}
    \State $\mathcal{L} \gets ||\boldsymbol{\epsilon}_3 + \frac{\mathbf{s}_{\theta}}{l_t}||^2$
    \State \textbf{Backpropagate} through $\mathcal{L}$  
    \EndFor
\end{algorithmic}
\end{algorithm}

TOLD and TOLD++ are structurally equivalent as it concerns Algorithm \ref{alg:TOLD}, except the difference between them is how the entries of the matrix $\exp(\mathbf{F}t)$, and $\Sigma_t$ are calculated. In addition to the performance benefit that TOLD++ provides, the method is also computationally cheaper. In TOLD, three exponentiations are required: $\exp(-t), \exp(-2t), \exp(-3t)$, whereas TOLD++ only requires computation of $\exp(-t\sqrt{3}), t^4, t^3, t^2$. Furthermore, the limits taken above indicate that both methods asymptotically reach the same latent distributions. This all means TOLD++ can do even more than TOLD with less computations, and still produce the same asymptotic results.

\section{Experiments and Results}
\label{sec:results}

This section provides several experiments, comparisons, and results that demonstrate the superior performance of TOLD++. The first experiment in Figure \ref{fig:convergence} was performed on a Gaussian Mixture Model to visualize the purpose of TOLD++. This figure simulates the forward SDEs in (\ref{eq:forwardsdes}) run with $50$ diffusion steps ($\Delta t = 1/50$), $T_{\text{final}} = 1$, $L = 1$, and with $1024$ samples from the distribution with pdf:
\[\pi(x) = 0.2 \cdot \mathcal{N}(x|0,0.5) + 0.4  \cdot \mathcal{N}(x|5, 1) + 0.4  \cdot \mathcal{N}(x|-5,1).\]

It is noticeable that the TOLD++ forward equations converge to a normal distribution faster than that of TOLD. Thus, they are a more efficient set of dynamics, visually validating the theoretical results.

To objectively validate the performance improvement, TOLD++ and TOLD are compared to each other on the Swiss Roll toy dataset, and the CIFAR-10 \cite{cifar10} dataset using the Frechet Inception Distance (FID) metric \cite{FIDmetric}. Since both methods converge to the same latent distribution, the performances are compared after an equal number of training iterations are completed. Two NVIDIA Tesla P100 GPUs were used for computation; a lack of newer GPUs resulted in challenges achieving results near the state-of-the-art.

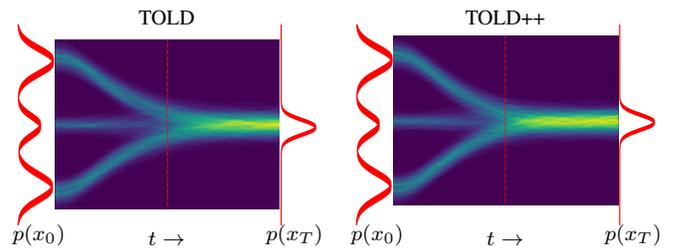
\begin{figure}[ht]
    \centering
    \begin{tikzpicture}
        % First image and its annotations
        \node at (0, 0) {\includegraphics[width=4cm]{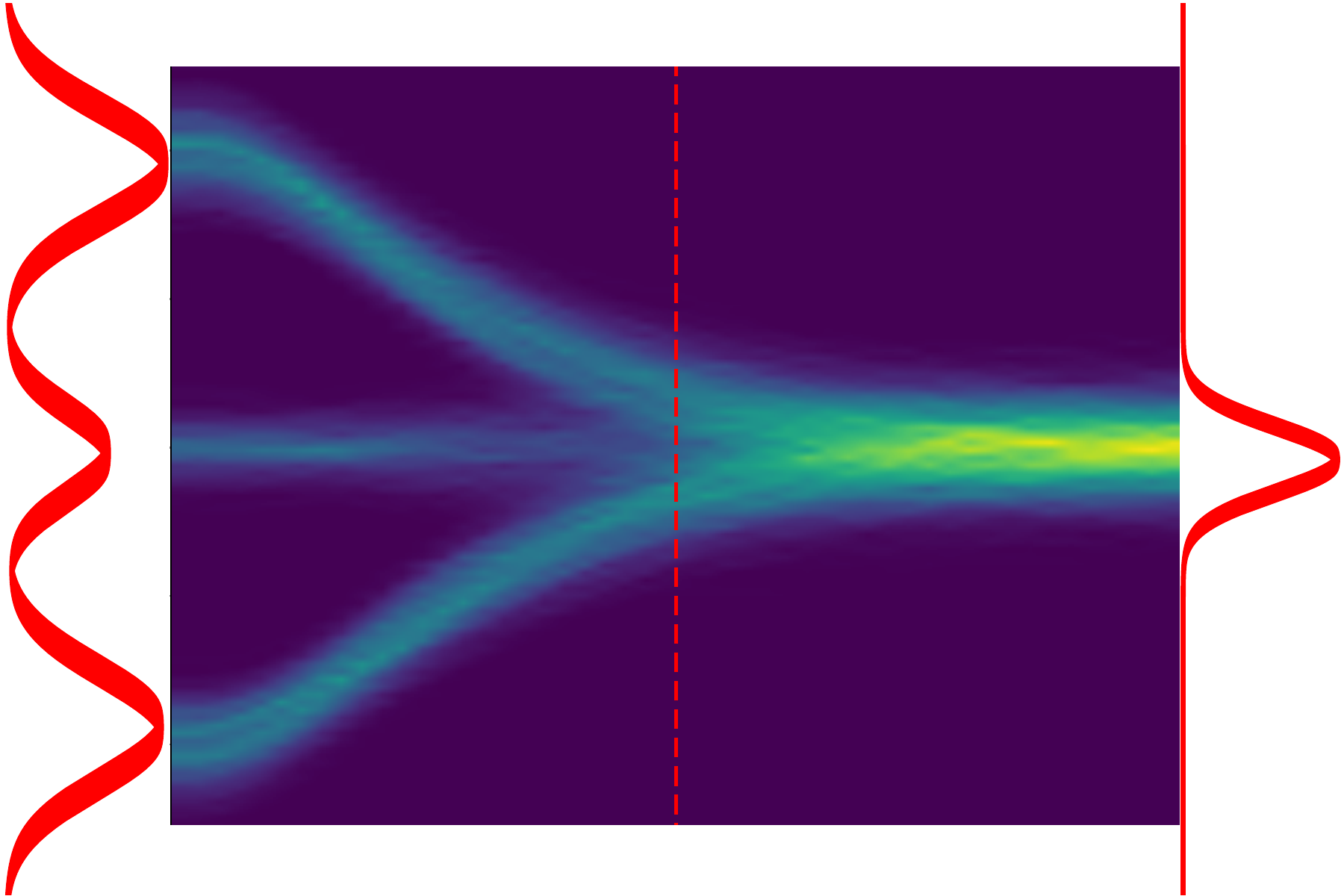}};
        
        % Overhead label for the first image
        \node[above, font=\footnotesize] at (0, 1.2) {TOLD};

        % Annotate the leftmost label "p(x_0)"
        \node[below, font=\footnotesize] at (-1.7, -1.2) {$p(x_0)$};

        % Annotate the rightmost label "p(x_T)"
        \node[below, font=\footnotesize] at (1.7, -1.2) {$p(x_T)$};

        % Label the x-axis with "Diffusion time t"
        \node[below, font=\footnotesize] at (0, -1.3) {$t \to$};

        % Second image and its annotations
        \node at (4.5, 0) {\includegraphics[width=4cm]{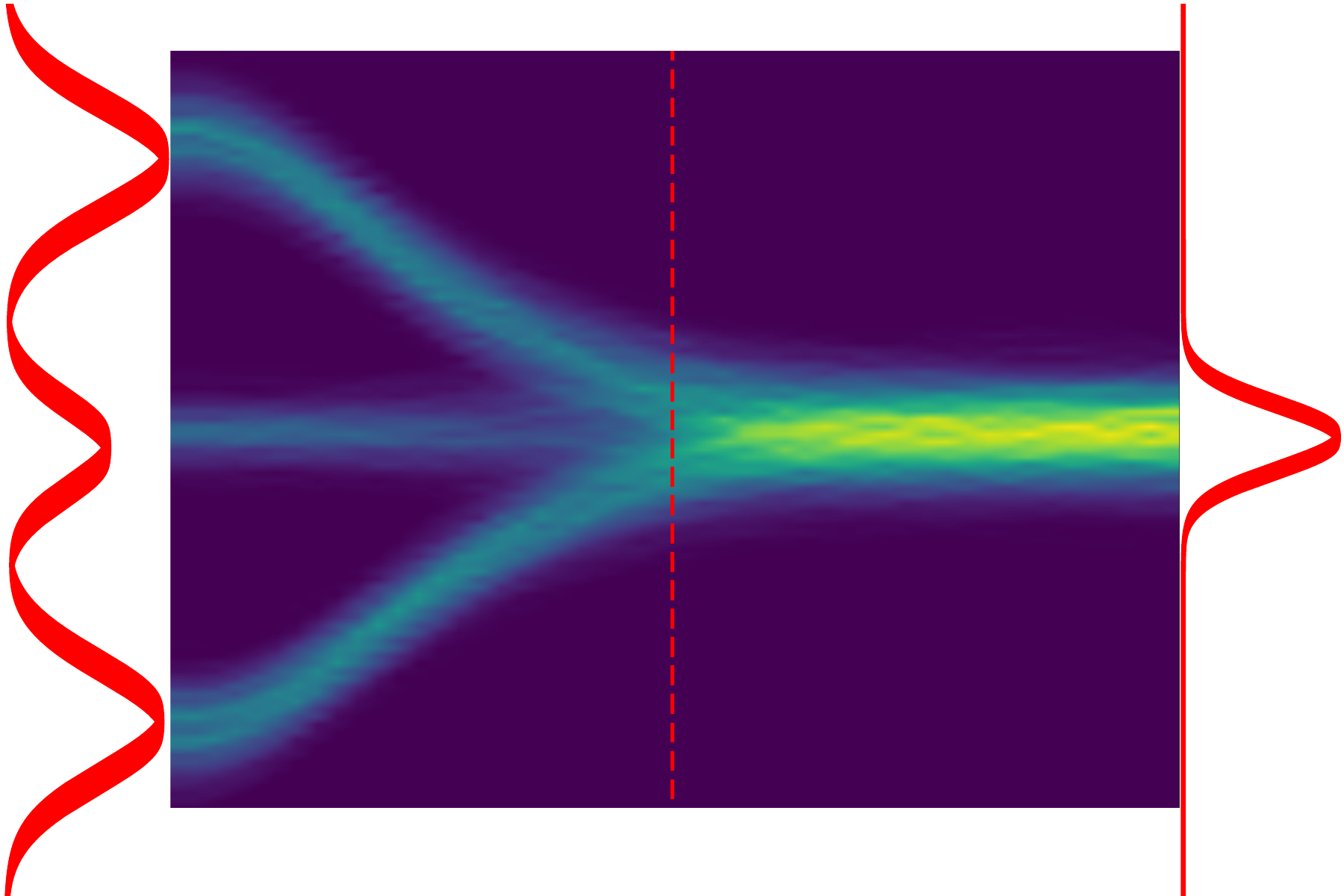}};
        
        % Overhead label for the second image
        \node[above, font=\footnotesize] at (4.5, 1.2) {TOLD++};

        % Annotate the leftmost label "p(x_0)"
        \node[below, font=\footnotesize] at (2.8, -1.2) {$p(x_0)$};

        % Annotate the rightmost label "p(x_T)"
        \node[below, font=\footnotesize] at (6.2, -1.2) {$p(x_T)$};

        % Label the x-axis with "Diffusion time t"
        \node[below, font=\footnotesize] at (4.5, -1.3) {$t \to$};

    \end{tikzpicture}
    \caption{Experiment on TOLD vs TOLD++ densities on samples from a Gaussian Mixture Model. The lighter the color, the more intense the density. The vertical dashed line occurs at the same diffusion time, and displays that TOLD++ converges faster. }
    \label{fig:convergence}
\end{figure}

\subsection{Swiss Roll dataset}

Table \ref{tab:fid_comparison} contains results for the Swiss Roll dataset. The experiments used a batch size of $2^{20} = $1,048,576 samples, and a fully connected Multi-Layer Perceptron (MLP) containing 7 input features (2 for position, velocity, and acceleration, each, and one additional for time), 7 hidden layers with 128 features per layer, and a 2-dimensional output layer; the SILU activation function was used for all nonlinearities. The parameters $L=4, T=1, \alpha=0.1$ were used along with a learning rate of $5\times10^{-3}$. For each row in Table \ref{tab:fid_comparison}, the network was trained from scratch 10 times, and for each network, 10 batches of samples (of size $2^{20}$) were taken. This produces a total of $100$ FID measurements per entry, of which the mean and standard deviation are reported. TOLD++ scores a statistically significant lower FID for every number of training steps, according to a Two-Sample t-Test, suggesting that it converges objectively faster than TOLD.

\subsection{CIFAR-10 dataset}

On the CIFAR-10 dataset, Table \ref{tab:fid_comparison_cifar} demonstrates that TOLD++ converges faster than TOLD for each number of training iterations. A total of 64 image samples of TOLD++ at 790,000 training iterations are given in Figure \ref{fig:cifar}. A Noise Conditional Score Network++ (NCSN++) network was used for score matching that contained 4 BigGAN type residue blocks, and a DDPM attention module with a resolution of 16 and training batch size of 64; we opted to use the same network that \cite{hold} did. The FIDs are calculated using the EM Method with 200 discrete steps and 10,000 samples, with an evaluation batch size of $16$. Due to computational constraints, there were challenges achieving the same FIDs \cite{hold} did, but nonetheless these methods are compared with identical networks and hyperparameters, illustrating the clear benefit of TOLD++.

\begin{figure}[!t]
    \centering
    \includegraphics[width=\linewidth]{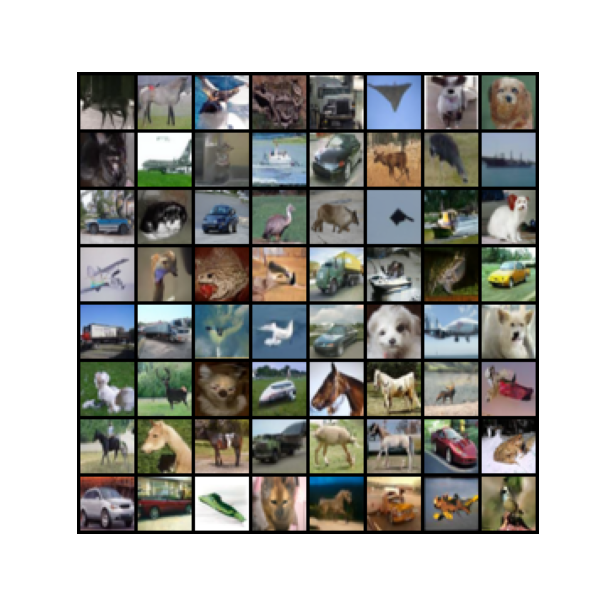}  % Include your image here
    \caption{Generated CIFAR-10 samples at 790,000 training iterations, without cherry picking.}
    \label{fig:cifar}
\end{figure}

\begin{table}[t!]
\centering
\caption{FID ($\downarrow$) Scores on Swiss Roll Dataset. Means reported, and standard deviations in parenthesis}
\begin{tabular}{ccc}
\toprule
\textbf{Training Iterations} & \textbf{TOLD} & \textbf{TOLD++} \\
\midrule
1000 & 2.681 (0.0337) & \textbf{2.566 (0.0541)} \\
2000 & 2.627 (0.0199) &  \textbf{2.439 (0.0201)} \\
3000 & 2.606 (0.0220) & \textbf{2.392 (0.0102)} \\
4000 & 2.584 (0.0102) & \textbf{2.381 (0.0145)} \\
5000 & 2.583 (0.0103) & \textbf{2.375 (0.0126)} \\
\bottomrule
\end{tabular}
\label{tab:fid_comparison}
\end{table}

\begin{table}[t!]
\centering
\caption{FID ($\downarrow$) Scores on CIFAR-10 Dataset.}
\begin{tabular}{ccc}
\toprule
\textbf{Training Iterations} & \textbf{TOLD} & \textbf{TOLD++} \\
\midrule
$3 \times 10^5$ & 8.232 & \textbf{7.518} \\
$4 \times 10^5$ & 7.900 &  \textbf{7.180}\\
$5 \times 10^5$ & 7.648 &  \textbf{6.818}\\
$6 \times 10^5$ & 7.520 &  \textbf{6.704}\\
$7 \times 10^5$ & 7.475 &  \textbf{6.498}\\
\bottomrule
\end{tabular}
\label{tab:fid_comparison_cifar}
\end{table}

\section{Conclusion}
\label{sec:conclusion}

This paper introduced TOLD++, a framework in which the optimal set of dynamics were derived under the original TOLD diffusion method. We provide theoretical guarantees that all systems governed by the considered models achieve optimal convergence speed using TOLD++. We demonstrate the validity of TOLD++ over TOLD visually as well as with the FID metric on both synthetic and real data, and showcase its promising performance. Future work on this topic naturally may ask the question, as to whether or not critical damping generalizes to Higher Order Langevin Dynamics of any order. Furthermore, the trade off between adding higher order dynamics, computational cost, and results still remains unexplored.

\bibliographystyle{IEEEtran}
\bibliography{refs}

\begin{thebibliography}{10}
\providecommand{\url}[1]{#1}
\csname url@samestyle\endcsname
\providecommand{\newblock}{\relax}
\providecommand{\bibinfo}[2]{#2}
\providecommand{\BIBentrySTDinterwordspacing}{\spaceskip=0pt\relax}
\providecommand{\BIBentryALTinterwordstretchfactor}{4}
\providecommand{\BIBentryALTinterwordspacing}{\spaceskip=\fontdimen2\font plus
\BIBentryALTinterwordstretchfactor\fontdimen3\font minus \fontdimen4\font\relax}
\providecommand{\BIBforeignlanguage}[2]{{%
\expandafter\ifx\csname l@#1\endcsname\relax
\typeout{** WARNING: IEEEtran.bst: No hyphenation pattern has been}%
\typeout{** loaded for the language `#1'. Using the pattern for}%
\typeout{** the default language instead.}%
\else
\language=\csname l@#1\endcsname
\fi
#2}}
\providecommand{\BIBdecl}{\relax}
\BIBdecl

\bibitem{sohl2015deep}
J.~Sohl-Dickstein, E.~Weiss, N.~Maheswaranathan, and S.~Ganguli, ``Deep unsupervised learning using nonequilibrium thermodynamics,'' in \emph{International Conference On Machine Learning}.\hskip 1em plus 0.5em minus 0.4em\relax PMLR, 2015, pp. 2256--2265.

\bibitem{normalizingflowsoriginal}
D.~Rezende and S.~Mohamed, ``Variational inference with normalizing flows,'' in \emph{International Conference On Machine Learning}.\hskip 1em plus 0.5em minus 0.4em\relax PMLR, 2015, pp. 1530--1538.

\bibitem{diffusiondenoising}
J.~Ho, A.~Jain, and P.~Abbeel, ``Denoising diffusion probabilistic models,'' \emph{Advances in Neural Information Processing Systems}, vol.~33, pp. 6840--6851, 2020.

\bibitem{jko}
C.~Xu, X.~Cheng, and Y.~Xie, ``Normalizing flow neural networks by {JKO} scheme,'' \emph{Advances in Neural Information Processing Systems}, vol.~36, 2024.

\bibitem{guth2022wavelet}
F.~Guth, S.~Coste, V.~De~Bortoli, and S.~Mallat, ``Wavelet score-based generative modeling,'' \emph{Advances in Neural Information Processing Systems}, vol.~35, pp. 478--491, 2022.

\bibitem{subspacediffusion}
B.~Jing, G.~Corso, R.~Berlinghieri, and T.~Jaakkola, \emph{Subspace Diffusion Generative Models}.\hskip 1em plus 0.5em minus 0.4em\relax Springer Nature Switzerland, 2022, p. 274–289.

\bibitem{diffusioncts}
Y.~Song, J.~Sohl-Dickstein, D.~P. Kingma, A.~Kumar, S.~Ermon, and B.~Poole, ``Score-based generative modeling through stochastic differential equations,'' \emph{arXiv preprint arXiv:2011.13456}, 2020.

\bibitem{dockhorn2021score}
T.~Dockhorn, A.~Vahdat, and K.~Kreis, ``Score-based generative modeling with critically-damped {L}angevin diffusion,'' \emph{arXiv preprint arXiv:2112.07068}, 2021.

\bibitem{doucet2022scorebased}
A.~Doucet, W.~Grathwohl, A.~G. Matthews, and H.~Strathmann, ``Score-based diffusion meets annealed importance sampling,'' \emph{Advances in Neural Information Processing Systems}, vol.~35, pp. 21\,482--21\,494, 2022.

\bibitem{hold}
Z.~Shi and R.~Liu, ``Generative modelling with higher-order {L}angevin dynamics,'' \emph{arXiv preprint arXiv:2404.12814}, 2024.

\bibitem{shi2024langwave}
------, ``Langwave: Realistic voice generation based on high-order {L}angevin dynamics,'' in \emph{ICASSP 2024-2024 IEEE International Conference on Acoustics, Speech and Signal Processing (ICASSP)}.\hskip 1em plus 0.5em minus 0.4em\relax IEEE, 2024, pp. 10\,661--10\,665.

\bibitem{shi2024noisy}
------, ``Noisy image restoration based on conditional acceleration score approximation,'' in \emph{ICASSP 2024-2024 IEEE International Conference on Acoustics, Speech and Signal Processing (ICASSP)}.\hskip 1em plus 0.5em minus 0.4em\relax IEEE, 2024, pp. 4000--4004.

\bibitem{ho2022classifierfreediffusionguidance}
\BIBentryALTinterwordspacing
J.~Ho and T.~Salimans, ``Classifier-free diffusion guidance,'' 2022. [Online]. Available: \url{https://arxiv.org/abs/2207.12598}
\BIBentrySTDinterwordspacing

\bibitem{anderson1982reverse}
B.~D. Anderson, ``Reverse-time diffusion equation models,'' \emph{Stochastic Processes and their Applications}, vol.~12, no.~3, pp. 313--326, 1982.

\bibitem{putzer}
E.~J. Putzer, ``Avoiding the {J}ordan canonical form in the discussion of linear systems with constant coefficients,'' \emph{The American Mathematical Monthly}, vol.~73, no.~1, pp. 2--7, 1966.

\bibitem{sarkka2019applied}
S.~Särkkä and A.~Solin, \emph{Applied Stochastic Differential Equations}.\hskip 1em plus 0.5em minus 0.4em\relax Cambridge University Press, 2019, vol.~10.

\bibitem{cifar10}
A.~Krizhevsky and G.~Hinton, ``Learning multiple layers of features from tiny images,'' University of Toronto, Tech. Rep., 2009.

\bibitem{FIDmetric}
M.~Heusel, H.~Ramsauer, T.~Unterthiner, B.~Nessler, and S.~Hochreiter, ``{GANS} trained by a two time-scale update rule converge to a local {N}ash equilibrium,'' \emph{Advances in Neural Information Processing Systems}, vol.~30, 2017.

\end{thebibliography}

\end{document}